\documentclass[journal]{IEEEtran}
\usepackage{amsmath,amsfonts,amssymb}
\usepackage{amsthm}
\usepackage{array}
\usepackage{graphicx}
\usepackage{booktabs}
\usepackage{url}
\usepackage{cite}
\newtheorem{proposition}{Proposition}
\theoremstyle{remark}
\newtheorem{remark}{Remark}

\begin{document}

\title{Direct Cartesian Integration vs.\ Filtered Back Projection: A Note on Direct Discretization and Variance Preservation}

\author{Mikhail~G.~Mozerov%
\thanks{M.~G.~Mozerov is with the Institute for Information Transmission Problems, Russian Academy of Sciences, Moscow, 127051 Russia (e-mail: mozer@iitp.ru).}}

\markboth{Technical Note}%
{Mozerov: Direct Cartesian Integration vs.\ Filtered Back Projection}

\maketitle

\begin{abstract}
This note compares two reconstruction operators for the parallel-beam inverse Radon problem, each obtained by the most direct discretization of a continuous inversion formula. The first is Filtered Back Projection (FBP), discretized with a sampled ramp filter and no zero-padding or empirical normalization. The second, termed Direct Cartesian Integration (DCI), evaluates the two-dimensional spectrum of the image point-wise on a Cartesian frequency grid by integrating each projection along the corresponding radial direction, using an identity that follows from the Central Slice Theorem and Hermitian symmetry. The comparison is deliberately restricted to these unmodified discretizations: its purpose is to expose properties of the discrete operators themselves, not to rank DCI against engineered FBP implementations. On a set of ten $512\times512$ test images with $800$ projections, DCI preserves the standard deviation of the reference image (mean ratio $0.997$, mean absolute deviation from unity $0.003$) and produces spatially uniform error maps, whereas the direct FBP discretization changes the standard deviation by $-38\%$ to $+20\%$ depending on the object and exhibits the radially structured ``cupping'' error. The observation is consistent with the fact that DCI never discretizes the polar Jacobian $|\omega|$, which is the source of the scale distortion in a direct ramp-filter discretization.
\end{abstract}

\begin{IEEEkeywords}
Computed tomography, inverse Radon transform, central slice theorem, direct Fourier reconstruction, filtered back projection, discretization, variance preservation, cupping.
\end{IEEEkeywords}

\section{Introduction}
\IEEEPARstart{T}{he} inversion of the Radon transform~\cite{radon1917} underlies computed tomography (CT)~\cite{hounsfield1973}. Of the many available methods~\cite{natterer2001,kak2001,beylkin2003,deans2007radon}, Filtered Back Projection (FBP)~\cite{ramachandran1971,shepp1974} remains the standard analytical algorithm. Its continuous formula follows from the Central Slice Theorem (CST)~\cite{bracewell1956} after a change of variables from Cartesian to polar frequency coordinates; the Jacobian of this change appears as the ramp filter $|\omega|$.

It is well known that a naive discretization of FBP is not accurate. When the ramp filter is sampled in the frequency domain of a length-$T$ discrete Fourier transform (DFT), its zero-frequency value is lost and the circular convolution produces interperiod interference; both effects create a low-frequency offset in the reconstructed image~\cite{kak2001}. Practical implementations therefore zero-pad each projection to at least twice its length and use a band-limited spatial-domain kernel, or calibrate the output empirically~\cite{kak2001}. These remedies are effective, and nothing in this note questions them.

The question addressed here is narrower. If one discretizes a continuous inversion formula in the most direct way, without padding and without calibration, which properties of the image survive? We compare two such direct discretizations. The first is the unmodified FBP described above. The second, which we call Direct Cartesian Integration (DCI), belongs to the family of direct Fourier methods~\cite{stark1981,goitein1972}: it assembles the 2D spectrum of the image on a Cartesian grid and inverts it with a 2D inverse FFT. Unlike interpolation-based direct Fourier methods, DCI does not resample 1D projection spectra in the frequency domain. Each Cartesian frequency sample is obtained by evaluating the Fourier integral of a projection at exactly the required radial frequency, with interpolation performed only along the angular coordinate of the sinogram. Related ideas of evaluating Fourier transforms at non-equispaced frequencies appear in non-equispaced FFT and polar FFT methods~\cite{fourmont2003,averbuch2006polar}; the present note does not propose a faster algorithm and makes no claim of novelty for the underlying identity, which is a restatement of the CST.

The workflows are contrasted in Fig.~\ref{fig:fig1}. The note makes three points:
\begin{itemize}
\item The identity behind DCI (Proposition~\ref{prop:dii}) expresses the 2D Fourier transform of a real function directly as a radial integral of its projections. In this form the polar Jacobian $|\omega|$ does not appear, so no ramp filter has to be discretized.
\item Discretized directly, DCI preserves the standard deviation of the reference image to within $1\%$ on all test images, while the directly discretized FBP does not.
\item The FBP error under direct discretization is spatially structured (cupping) and is not removed by matching the mean and standard deviation after reconstruction, whereas the DCI error is spatially uniform.
\end{itemize}

\begin{figure}[!t]
\centering
\includegraphics[width=0.90\columnwidth]{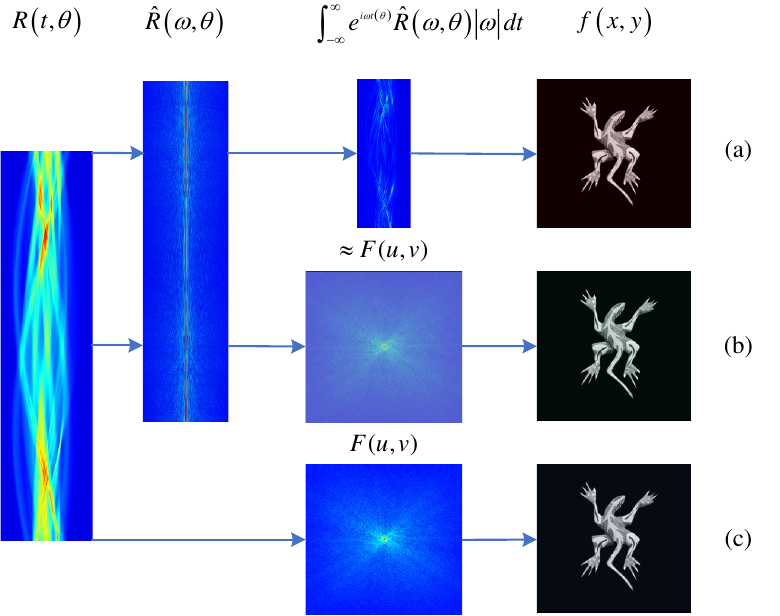}
\caption{Reconstruction workflows: (a) filtered back projection (FBP); (b) interpolation-based direct Fourier reconstruction; (c) direct Cartesian integration (DCI), in which the 2D spectrum is computed point-wise from the sinogram.}
\label{fig:fig1}
\end{figure}

\section{The Direct Integration Identity}
Let $f(x,y)$ be a real function on $\mathbb{R}^2$ that decays sufficiently fast at infinity. Its Radon transform is
\begin{equation}
\label{radon}
R(t,\theta) = \int_{-\infty}^{\infty} f\bigl(t\cos\theta - z\sin\theta,\; t\sin\theta + z\cos\theta\bigr)\,dz,
\end{equation}
that is, the integral of $f$ along the line $x\cos\theta + y\sin\theta = t$. With the 2D Fourier transform
\begin{equation}
\label{fourier}
F(u,v) = \iint_{\mathbb{R}^2} f(x,y)\, e^{-i(ux+vy)}\,dx\,dy,
\end{equation}
and its inverse
\begin{equation}
\label{inv_fourier}
f(x,y) = \frac{1}{4\pi^2}\iint_{\mathbb{R}^2} F(u,v)\, e^{i(ux+vy)}\,du\,dv,
\end{equation}
the CST states that
\begin{equation}
\label{CST}
\hat{R}(\omega, \theta) = \int_{-\infty}^{\infty} R(t, \theta)\, e^{-i\omega t}\, dt = F(\omega \cos \theta, \omega \sin \theta).
\end{equation}
Substituting~\eqref{CST} into~\eqref{inv_fourier} in polar coordinates $(\omega,\theta)$ gives the FBP formula
\begin{equation}
\label{recon}
f(x,y) = \frac{1}{4\pi^2}\int_{0}^{\pi}\!\!\int_{-\infty}^{\infty} \hat R(\omega,\theta)\,|\omega|\, e^{i\omega(x\cos\theta + y\sin\theta)}\, d\omega\, d\theta,
\end{equation}
in which the ramp filter $|\omega|$ is the Jacobian of the change of variables.

Instead of changing variables in~\eqref{inv_fourier}, one may evaluate $F$ directly at Cartesian points. Define the direct integration function
\begin{equation}
\label{DIF}
D(u,v) = \begin{cases}
\displaystyle\int_{-\infty}^{\infty} R\bigl(t,\operatorname{atan2}(v,u)\bigr)\, e^{-i\sqrt{u^2+v^2}\,t}\,dt, \\
\hfill \text{if } v>0 \lor (v=0 \land u>0); \\[1.2ex]
\displaystyle\int_{-\infty}^{\infty} R(t,\theta_0)\,dt, \hfill \text{if } u=v=0; \\[1.2ex]
D^{*}(-u,-v), \hfill \text{otherwise},
\end{cases}
\end{equation}
where $\theta_0\in[0,\pi)$ is arbitrary and $(\cdot)^*$ denotes complex conjugation.

\begin{proposition}[Direct Integration Identity]
\label{prop:dii}
If $f$ is real, then $D(u,v) = F(u,v)$ for all $(u,v)\in\mathbb{R}^2$.
\end{proposition}

\begin{proof}
\emph{Case 1:} $v>0$, or $v=0$ and $u>0$. Let $\omega=\sqrt{u^2+v^2}>0$ and $\theta=\operatorname{atan2}(v,u)\in[0,\pi)$. Then $u=\omega\cos\theta$, $v=\omega\sin\theta$, and by~\eqref{CST}, $D(u,v)=\hat R(\omega,\theta)=F(u,v)$.

\emph{Case 2:} $(u,v)=(0,0)$. The integral of any projection over $t$ equals $\iint f\,dx\,dy = F(0,0)$, independently of $\theta_0$.

\emph{Case 3:} all remaining points. Then $(-u,-v)$ falls under Case~1, so $D(-u,-v)=F(-u,-v)$. Since $f$ is real, $F(-u,-v)=F^*(u,v)$, hence $D(u,v)=D^*(-u,-v)=F(u,v)$.
\end{proof}

Consequently,
\begin{equation}
\label{inv_D}
f(x,y) = \frac{1}{4\pi^2}\iint_{\mathbb{R}^2} D(u,v)\, e^{i(ux+vy)}\,du\,dv .
\end{equation}

\begin{remark}
Proposition~\ref{prop:dii} is a reformulation of the CST in Cartesian frequency coordinates. Hermitian symmetry is used to restrict the angular argument to $[0,\pi)$, which is the range covered by parallel-beam data, and halves the amount of computation. The practical significance lies in~\eqref{inv_D}: the inverse transform is taken in the same Cartesian coordinates in which $D$ is evaluated, so no Jacobian weighting arises, and every grid point is computed independently (Fig.~\ref{fig:discrete_grid}). In~\eqref{recon}, in contrast, the weight $|\omega|$ must be discretized, and this is where a direct discretization of FBP loses the zero-frequency information.
\end{remark}

\begin{figure}[!t]
\centering
\includegraphics[width=0.78\columnwidth]{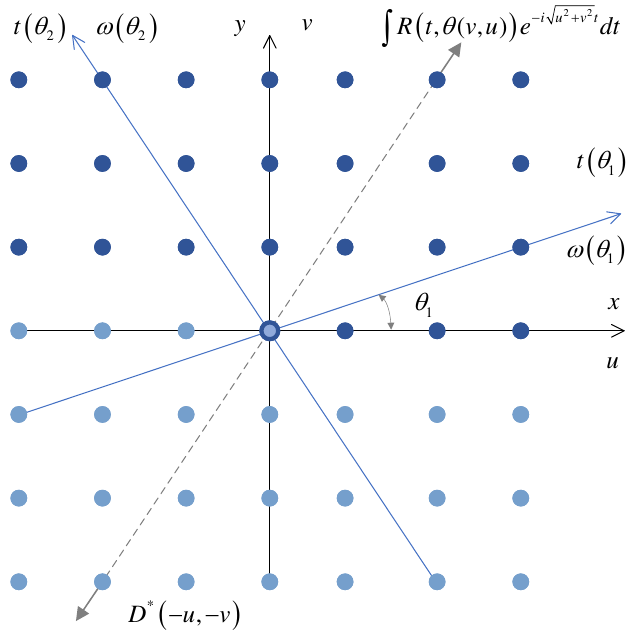}
\caption{Evaluation of $D(u,v)$ on a Cartesian grid. Each point in the upper half-plane is obtained from the projection at angle $\operatorname{atan2}(v,u)$ at radial frequency $\sqrt{u^2+v^2}$; the lower half-plane follows from Hermitian symmetry.}
\label{fig:discrete_grid}
\end{figure}

\section{Direct Discretization}
\label{sec:method}
\subsection{DCI}
Let the projections be sampled with unit detector step on $t=-T/2,\dots,T/2-1$, where $T$ is also the side of the reconstructed $T\times T$ image, and let $\mathbf{R}_{t,a}=R(t,\varphi_a)$ with $\varphi_a=a\pi/A$, $a=0,\dots,A-1$. With $u=2\pi n/T$, $v=2\pi m/T$ and $-T/2\le n,m\le T/2-1$, the direct discretization of~\eqref{DIF} is
\begin{equation}
\label{DIF_discrete}
\mathbf{D}_{n,m} = \begin{cases}
\displaystyle\sum_{t=-T/2}^{T/2-1} \mathbf{R}(t,\theta_{n,m})\, e^{-i\frac{2\pi}{T}\sqrt{n^2+m^2}\,t}, \\
\hfill \text{if } m>0 \lor (m=0 \land n>0); \\[1.2ex]
\displaystyle\frac{1}{A}\sum_{a=0}^{A-1}\sum_{t=-T/2}^{T/2-1} \mathbf{R}_{t,a}, \hfill \text{if } n=m=0; \\[1.2ex]
\mathbf{D}^*_{-n,-m}, \hfill \text{otherwise},
\end{cases}
\end{equation}
where $\theta_{n,m}=\operatorname{atan2}(m,n)$. The zero-frequency term is, in the continuous setting, the sum of any single projection; averaging over all $A$ projections reduces its sensitivity to discretization error. Values at the angles $\theta_{n,m}$, which generally fall between acquired angles, are obtained by angular interpolation in the sinogram,
\begin{equation}
\label{interpolation_formula}
\mathbf{R}(t,\theta) \approx \sum_{a=0}^{A-1} w_a(\theta)\, \mathbf{R}_{t,a},
\end{equation}
with weights $w_a$ of a nearest-neighbor, linear or cubic kernel. The image is then
\begin{equation}
\label{inv_DFT_sym}
\mathbf{f}_{j,k} = \frac{1}{T^2} \sum_{n=-T/2}^{T/2-1}\sum_{m=-T/2}^{T/2-1} \mathbf{D}_{n,m}\, e^{i\frac{2\pi}{T}(nj+mk)},
\end{equation}
computed by a 2D inverse FFT after the circular shift $n'=(n+T/2)\bmod T$ (and likewise for $m$) that moves the zero frequency to the first array element. The algorithm is:
\begin{enumerate}
\item Center the sinogram on $t\in[-T/2,T/2-1]$; no padding is applied, so the detector length equals $T$.
\item For each $(n,m)$ in the upper half-plane, interpolate the projection at $\theta_{n,m}$ by~\eqref{interpolation_formula} and evaluate the radial sum in~\eqref{DIF_discrete}.
\item Set $\mathbf{D}_{0,0}$ from the projection sums and fill the lower half-plane by Hermitian symmetry.
\item Apply the circular shift and the 2D inverse FFT~\eqref{inv_DFT_sym}, then shift back to centered coordinates.
\end{enumerate}
Evaluated directly, step~2 costs $O(T)$ per frequency sample and $O(T^3)$ in total, independently of $A$; a $512\times512$ image is reconstructed in under 2~s on a standard laptop. The radial sums could be accelerated with non-equispaced FFTs~\cite{fourmont2003}, but this is not pursued here.

Two properties are worth noting. First, there is no step at which a weighting function has to be sampled, and the zero-frequency term is taken directly from the data. Second, the only periodic operation is the final 2D inverse DFT, which is harmless when the object lies inside the $T\times T$ field.

\begin{remark}
In the corners of the Cartesian spectrum, $\sqrt{n^2+m^2}>T/2$, the radial frequency exceeds the Nyquist frequency of the detector sampling. Those samples are therefore affected by aliasing of the projection data. This limitation is inherent in any square-grid reconstruction from projections sampled with the same step, and it contributes to the residual error of DCI.
\end{remark}

\subsection{Directly Discretized FBP}
The FBP baseline is the direct discretization of~\eqref{recon}: each projection of length $T$ is transformed by a length-$T$ DFT, multiplied by the sampled ramp $|k|$, transformed back without zero-padding or apodization, and back-projected with cubic convolution interpolation~\cite{unser1999splines}. In this form the sampled ramp vanishes at $k=0$ and the filtering is a circular convolution of length $T$, so the reconstructed image suffers a low-frequency offset~\cite{kak2001}. The offset is largest when the projections fill the whole detector, because the wrap-around of the ramp kernel then affects every sample. Zero-padding to at least $2T$ (doubling the projection length) together with a spatial-domain ramp kernel removes this effect~\cite{kak2001}; the baseline intentionally omits these steps.

\section{Numerical Results}
\subsection{Setup}
Sinograms of ten $512\times512$ test images (Fig.~\ref{fg:images}) were generated with a discrete projection operator~\cite{herman2009fundamentals} using bicubic interpolation, with $A=800$ projections over $[0,\pi)$, which approximates the $\frac{\pi}{2}T$ angular sampling recommended in~\cite{herman2009fundamentals}. All data are noise-free. For DCI, the angular interpolation kernel has a moderate effect: the average PSNR over the ten images is $37.86$, $39.40$ and $39.67$~dB for nearest-neighbor, linear and cubic interpolation, respectively. Cubic interpolation is used below.

\begin{figure*}[!t]
\centering
\setlength{\tabcolsep}{2pt}
\begin{tabular}{ccccc}
\includegraphics[width=0.36\columnwidth]{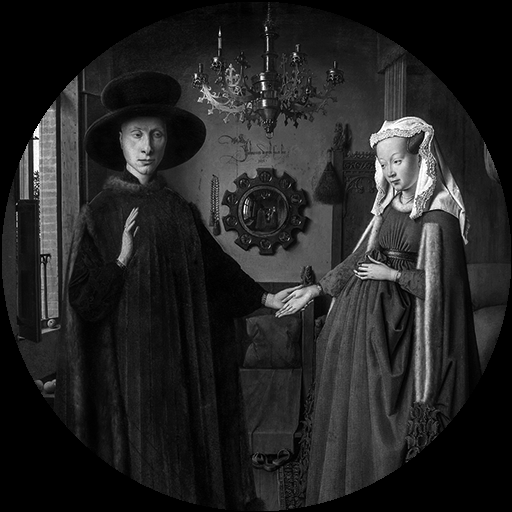} &
\includegraphics[width=0.36\columnwidth]{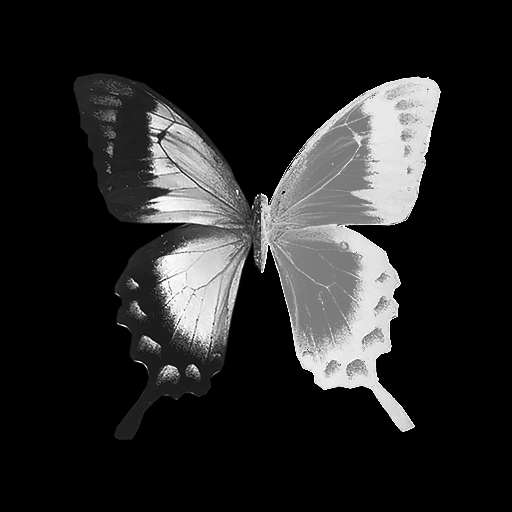} &
\includegraphics[width=0.36\columnwidth]{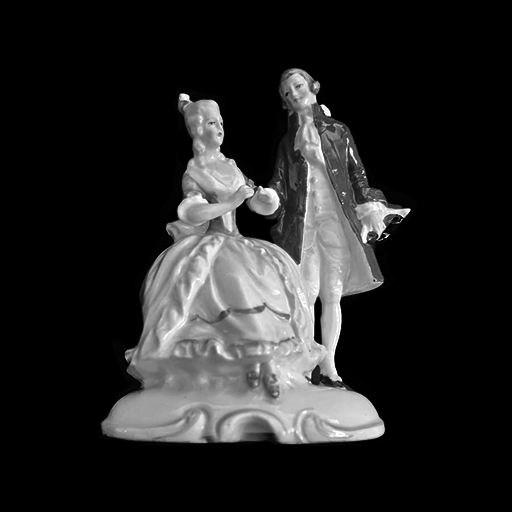} &
\includegraphics[width=0.36\columnwidth]{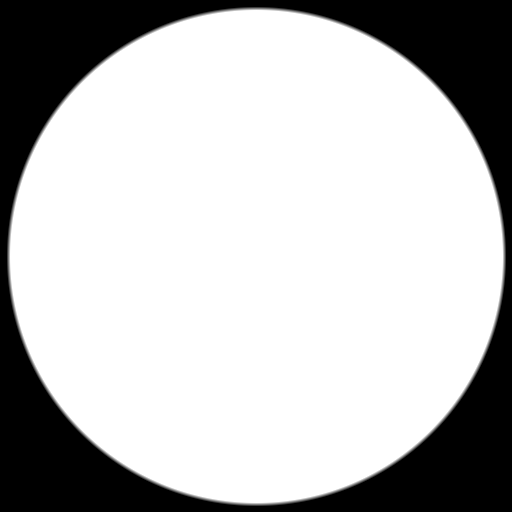} &
\includegraphics[width=0.36\columnwidth]{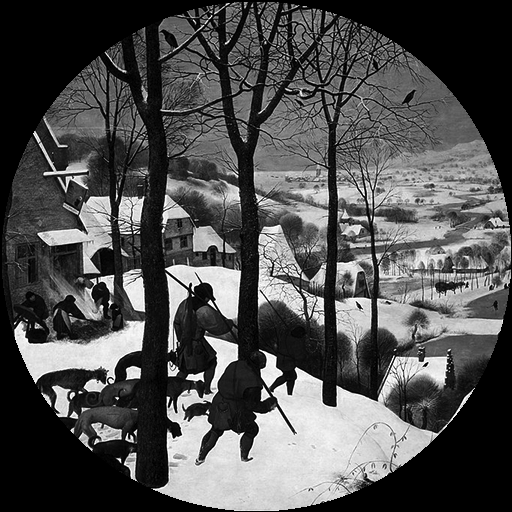} \\
Arnolfini & Butterfly & Dance & Disk & Hunters\\
\includegraphics[width=0.36\columnwidth]{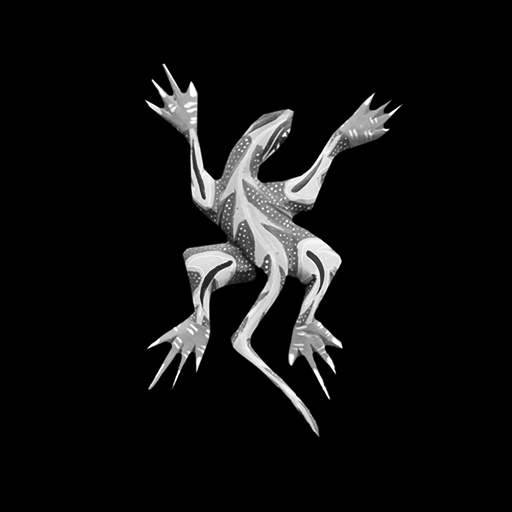} &
\includegraphics[width=0.36\columnwidth]{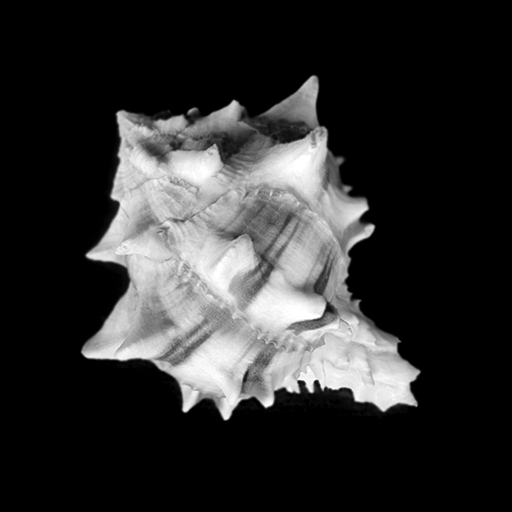} &
\includegraphics[width=0.36\columnwidth]{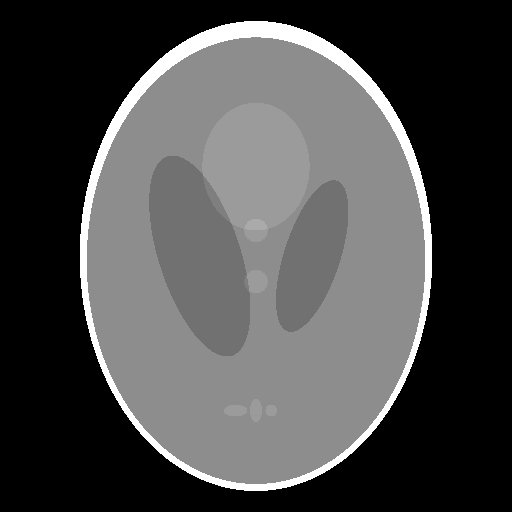} &
\includegraphics[width=0.36\columnwidth]{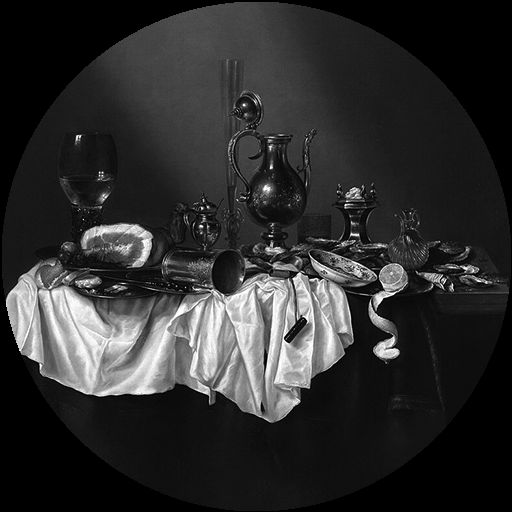} &
\includegraphics[width=0.36\columnwidth]{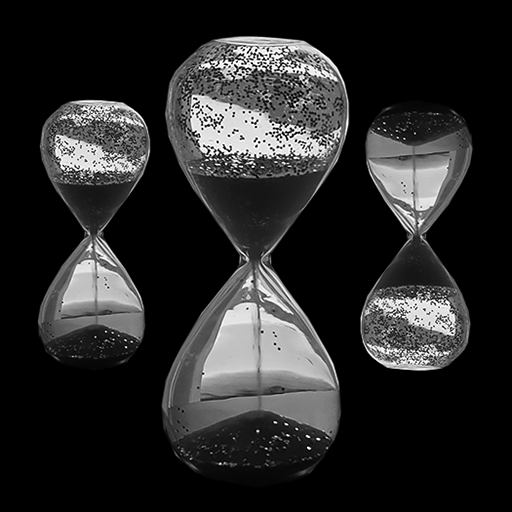} \\
Lizard & Seashell & Shepp-Logan & Still-life & Time\\
\end{tabular}
\caption{The ten $512\times512$ test images.}
\label{fg:images}
\end{figure*}

\subsection{Variance Preservation}
The main quantity of interest is the standard deviation ratio $\mathrm{SDR}=\sigma_{\mathrm{rec}}/\sigma_{\mathrm{ref}}$, where $\sigma_{\mathrm{rec}}$ and $\sigma_{\mathrm{ref}}$ are the empirical standard deviations of the reconstructed and reference images. The SDR is unaffected by any additive offset, so it isolates the scale behavior of the operator from the loss of the mean.

Table~\ref{tab:sdr_only} shows that DCI keeps the SDR between $0.991$ and $1.000$ for every image, with a mean absolute deviation from unity of $0.003$. The small residual is consistent with the approximations in the forward projector and the angular interpolation. By Parseval's theorem, the variance of the image equals the energy of its spectrum outside the origin; since DCI estimates each Cartesian spectral sample directly from the data, this energy is retained without any normalization step.

The directly discretized FBP gives SDR values from $0.615$ to $1.196$. Its average, $1.016$, is close to unity only because errors of opposite sign cancel; the mean absolute deviation is $0.132$. The largest error occurs for the Disk, whose projections occupy the full detector width, which agrees with the wrap-around mechanism described in Section~\ref{sec:method}.

\begin{table}[!t]
\centering
\footnotesize
\caption{Standard Deviation Ratio $\sigma_{\mathrm{rec}}/\sigma_{\mathrm{ref}}$ for Noise-Free Data ($A=800$).}
\label{tab:sdr_only}
\begin{tabular*}{\columnwidth}{@{\extracolsep{\fill}} l cc @{}}
\toprule
Image & DCI & Direct FBP \\
\midrule
Arnolfini   & 0.994 & 1.061 \\
Butterfly   & 0.997 & 1.116 \\
Dance       & 1.000 & 1.107 \\
Disk        & 1.000 & 0.615 \\
Hunters     & 0.991 & 0.925 \\
Lizard      & 0.998 & 1.196 \\
Seashell    & 1.000 & 1.094 \\
Shepp-Logan & 0.998 & 0.878 \\
Still-life  & 0.997 & 1.111 \\
Time        & 0.995 & 1.054 \\
\midrule
Mean SDR                     & 0.997 & 1.016 \\
Mean $|\mathrm{SDR}-1|$      & 0.003 & 0.132 \\
\bottomrule
\end{tabular*}
\end{table}

\subsection{Reconstruction Fidelity}
Because the direct FBP loses the zero-frequency term, raw PSNR values would mostly reflect a constant offset. To look past it, two post-hoc corrections of the FBP output are reported. In FBP-M, the image mean is set to the value available from the data, $\mathbf{D}_{0,0}/T^2$. In FBP-MS, the image is in addition rescaled so that its standard deviation matches that of the reference. Neither correction changes the spatial distribution of the error. Table~\ref{tab:dit_fbp_comparison} lists PSNR and SSIM~\cite{wang2004ssim} for four representative images and the average over all ten.

DCI reaches an average PSNR of $39.67$~dB and SSIM of $0.999$. The corrected FBP variants reach $24.24$ and $27.12$~dB, respectively. Matching the first two moments improves FBP on some images, most notably Lizard, where FBP-MS reaches the same SSIM as DCI, but it degrades the Disk and leaves a large gap on average. A global mean and scale correction therefore does not compensate the error of the direct FBP discretization, which indicates that the error is spatially structured.

\begin{table}[!t]
\centering
\footnotesize
\setlength{\tabcolsep}{3pt}
\caption{PSNR (dB) and SSIM for Noise-Free Data ($A=800$). Average Is Over All Ten Images.}
\label{tab:dit_fbp_comparison}
\begin{tabular*}{\columnwidth}{@{\extracolsep{\fill}} l cc cc cc @{}}
\toprule
& \multicolumn{2}{c}{DCI} & \multicolumn{2}{c}{FBP-M} & \multicolumn{2}{c}{FBP-MS} \\
\cmidrule(lr){2-3} \cmidrule(lr){4-5} \cmidrule(lr){6-7}
Image & PSNR & SSIM & PSNR & SSIM & PSNR & SSIM \\
\midrule
Disk        & 50.32 & 0.999 & 13.41 & 0.759 & 11.62 & 0.850 \\
Hunters     & 30.14 & 0.996 & 18.75 & 0.928 & 18.27 & 0.931 \\
Lizard      & 42.86 & 0.999 & 26.76 & 0.983 & 38.88 & 0.999 \\
Shepp-Logan & 34.21 & 0.998 & 25.22 & 0.984 & 27.13 & 0.992 \\
\midrule
Average     & 39.67 & 0.999 & 24.24 & 0.954 & 27.12 & 0.968 \\
\bottomrule
\end{tabular*}
\end{table}

\subsection{Spatial Error Distribution}
Fig.~\ref{fg:err_map} shows signed error maps. For the direct FBP, the error is dominated by a smooth, radially varying component: it is small near the center, grows in magnitude toward the boundary of the object and changes sign outside it. This low-frequency shading is the cupping pattern. Because this component varies slowly, the FBP images can look nearly correct even though the error is large. For DCI, the error has no visible radial trend and is concentrated at edges and fine detail, where it is attributable to interpolation and sampling.

\begin{figure}[!t]
\centering
\setlength{\tabcolsep}{2pt}
\begin{tabular}{cc}
\includegraphics[width=0.42\columnwidth]{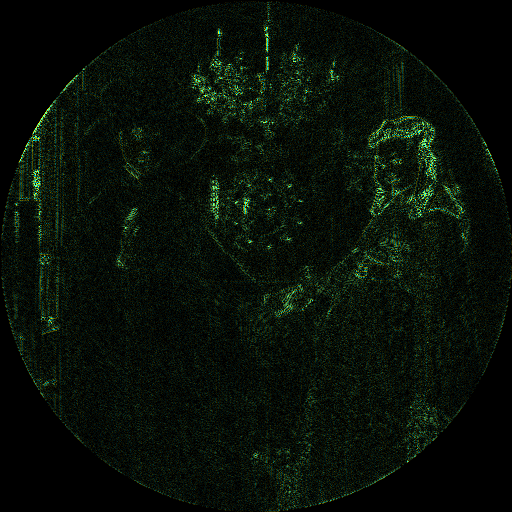} &
\includegraphics[width=0.42\columnwidth]{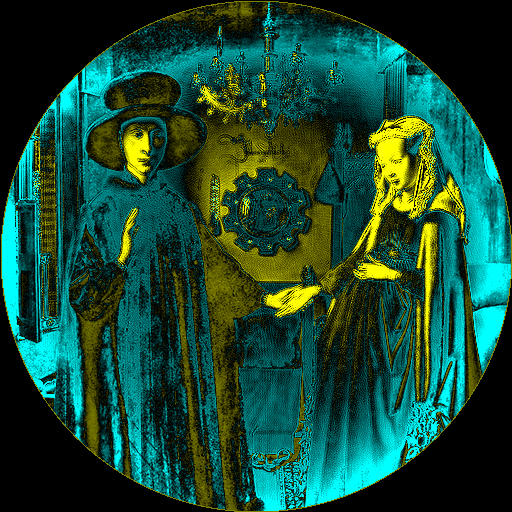} \\
\includegraphics[width=0.42\columnwidth]{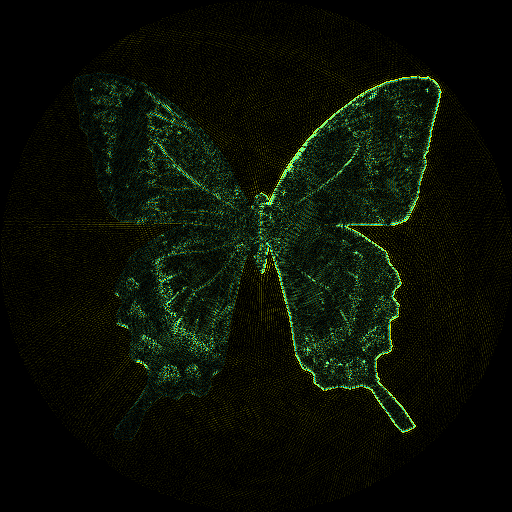} &
\includegraphics[width=0.42\columnwidth]{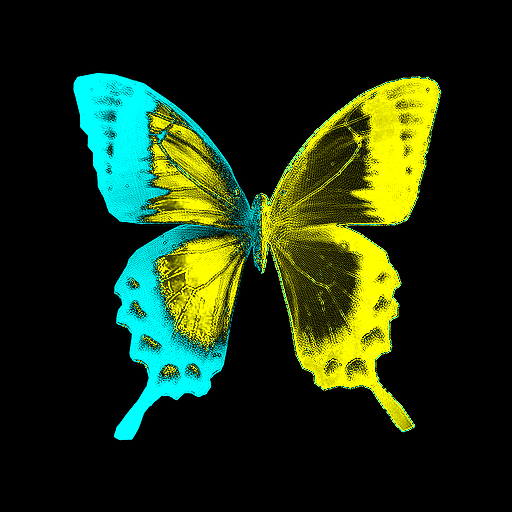} \\
\includegraphics[width=0.42\columnwidth]{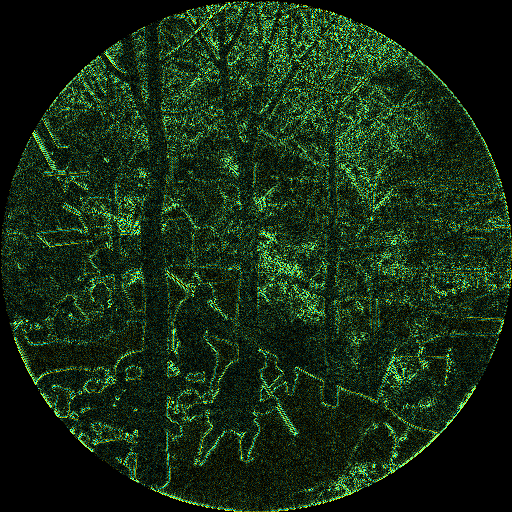} &
\includegraphics[width=0.42\columnwidth]{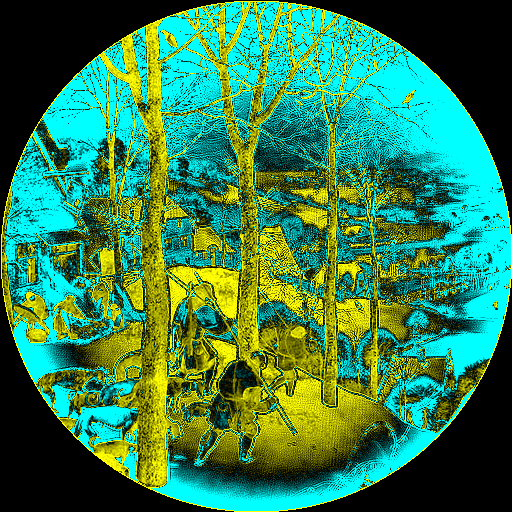} \\
\includegraphics[width=0.42\columnwidth]{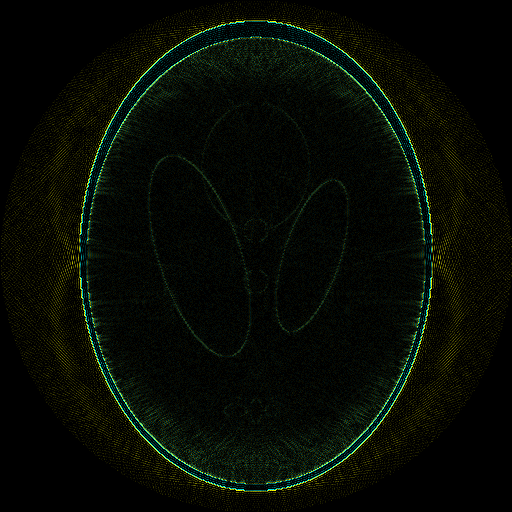} &
\includegraphics[width=0.42\columnwidth]{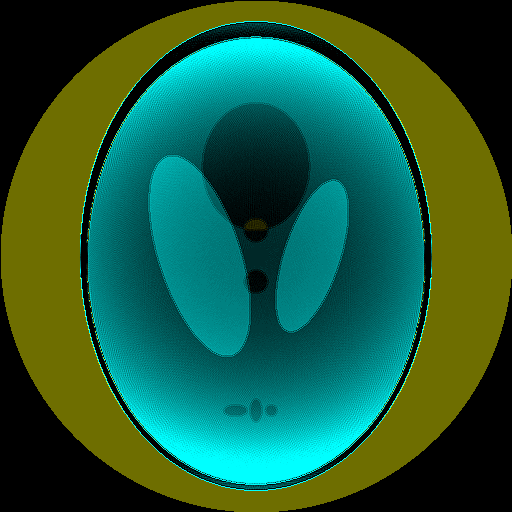} \\
\includegraphics[width=0.42\columnwidth]{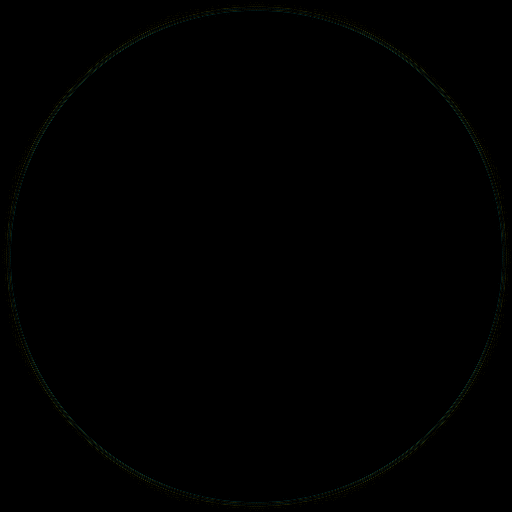} &
\includegraphics[width=0.42\columnwidth]{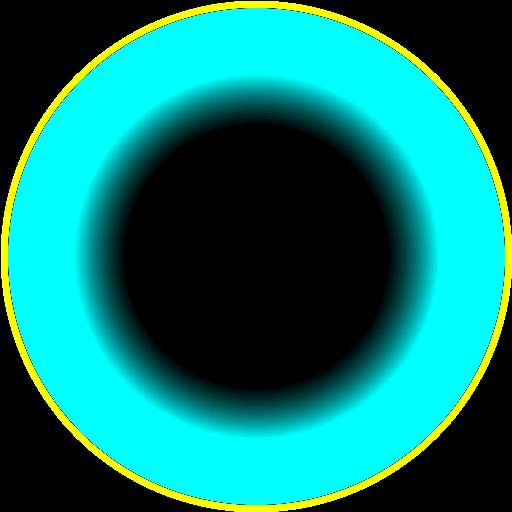} \\
DCI & Direct FBP \\
\end{tabular}
\caption{Signed reconstruction error maps at $A=800$ for Arnolfini, Butterfly, Hunters, Shepp-Logan and Disk (top to bottom). Yellow marks positive and cyan negative deviations from the reference. The direct FBP error (right) contains a smooth radial (cupping) component; the DCI error (left) does not.}
\label{fg:err_map}
\end{figure}

\subsection{Scope}
The results characterize two direct discretizations and nothing more. FBP with zero-padding to at least twice the projection length and a spatial-domain ramp kernel is known to remove the zero-frequency loss and the cupping it causes~\cite{kak2001}, and this note makes no claim about the accuracy of DCI relative to such implementations, nor relative to iterative or regularized methods. Noisy data, sparse angular sampling and divergent-beam geometries are outside the scope of this note.

\section{Conclusion}
For the parallel-beam inverse Radon problem, the identity $D(u,v)=F(u,v)$ allows the 2D spectrum to be evaluated directly on a Cartesian grid, using Hermitian symmetry to cover the full frequency plane from projections over $[0,\pi)$. Discretized without any additional steps, the resulting DCI operator yields statistically consistent reconstructions: the standard deviation of the reference image is preserved to within $1\%$, the zero-frequency term is taken from the data, and the error map shows no radial structure. These properties follow from the formulation itself and require no heuristics. The classical FBP formula, discretized in the same direct manner, distorts the image scale and produces cupping, because its ramp weighting and circular filtering lose low-frequency information at the boundaries of the data. In practice FBP removes these distortions through engineering modifications, chiefly zero-padding the projections to double their length. DCI does not need such modifications to achieve variance preservation and a uniform error distribution.

\section*{Acknowledgments}
The author thanks Daniil M. Mozerov for his assistance with the numerical verification and for helpful discussions.

\bibliographystyle{IEEEtran}
\bibliography{dit}

\end{document}